\documentclass[11pt]{article}
\usepackage{epsfig,makeidx,amsmath,amsfonts,latexsym,subfigure}
\usepackage{times,pifont}
\usepackage[utf8]{inputenc}
\usepackage{booktabs}       
\usepackage{graphicx}
\usepackage{bm}
\usepackage{xcolor}
\usepackage{enumitem}

\def\smallpar{\smallbreak\@afterindentfalse\@afterheading\ignorespaces}

\newtheorem{theorem}{Theorem}[section]
\newtheorem{proposition}[theorem]{Proposition}
\newtheorem{lemma}[theorem]{Lemma}
\newtheorem{corollary}[theorem]{Corollary}
\newtheorem{definition}[theorem]{Definition}

\newenvironment{proof}{{\em Proof. }}{\hfill $\Box$ \vspace{1em}}

\usepackage{xcolor}

\newcommand{\E}{{\mathbb E}}

\newcommand{\Var}{{\mathbb V}{\rm ar}}
\newcommand{\Cov}{{\mathbb C}{\rm ov}}

\newcommand{\bb}{{\mathbf b}}
\newcommand{\xx}{\mathbf{x}}
\newcommand{\yy}{{\mathbf y}}

\newcommand{\bz}{{\mathbf 0}}

\newcommand{\btheta}{\bm{\theta}}

\newcommand{\bomega}{\bm{\omega}}
\newcommand{\bta}{\bm{\eta}}

\newcommand{\Pro}{{\mathbb P}}
\newcommand{\Qro}{\mathbb Q}

\newcommand{\bi}{\begin{itemize}}
\newcommand{\ei}{\end{itemize}}
\newcommand{\ra}{\rightarrow}

\newcommand{\ep}{\epsilon}

\newcommand{\sign}{{\operatorname{sign}}}

\newcommand{\iy}{\infty}

\newcommand{\beq}{\begin{equation}}
\newcommand{\bsq}{\end{equation}}
\newcommand{\beqa}{\begin{eqnarray*}}
\newcommand{\bsqa}{\end{eqnarray*}}
\newcommand{\btm}{\begin{theorem}}
\newcommand{\etm}{\end{theorem}}
\newcommand{\bpf}{\begin{proof}}
\newcommand{\epf}{\end{proof}}
\newcommand{\bla}{\begin{lemma}}
\newcommand{\ela}{\end{lemma}}
\newcommand{\bdn}{\begin{definition}}
\newcommand{\edn}{\end{definition}}
\newcommand{\bpn}{\begin{proposition}}
\newcommand{\epn}{\end{proposition}}
\newcommand{\bcy}{\begin{corollary}}
\newcommand{\ecy}{\end{corollary}}

\def\ldotsplus{\mathinner{\ldotp\ldotp\ldotp\ldotp}}
\def\fourdots{\relax\ifmmode\ldotsplus\else$\m@th \ldotsplus\,$\fi}

\begin{document}
\title{Annealed correlation and quenched noise sensitivity in biased deep neural networks}
\author{Johan Jonasson\thanks{Chalmers University of Technology and University of Gothenburg, S-412 96 Gothenburg, Sweden, jonasson@chalmers.se  }
\thanks{Research partly supported by WASP (Wallenberg Autonomous Systems and Software Program)}}
\maketitle
\begin{abstract}

In \cite{JSZ}, fully connected deep neural networks with layers of width $n$, i.i.d.\ standard Gaussian weights, sign activations and no bias terms were shown to be annealed and quenched noise sensitive with respect to any joint probability distribution $\Qro^{(n)}$ on the input $\bomega$ and its perturbation $\bta$ satisfying a mild positive-correlation condition and $\Qro^{(n)}(\bta \in \{\bomega,-\bomega\}) \to 0$. 
Here we add independent Gaussian biases of the natural order and show that quenched noise sensitivity is retained over a broad range of depths, indeed under the slightly weaker assumption that $\Qro^{(n)}(\bomega = \bta) \to 0$. 
Thus, subject to the positive-correlation condition, quenched noise sensitivity holds under the minimal requirement that the perturbed input is asymptotically unlikely to coincide with the original input.
At the same time, the annealed covariance remains bounded away from zero and is entirely accounted for by fluctuations in the network-dependent output means.

\end{abstract}

\section{Introduction}

This paper addresses a question left open in \cite{JSZ}. 
The driving force there was to investigate how sensitive binary deep neural net (DNN) classifiers are to input noise, i.e.\ if small changes in a typical input make the classifier change its mind. 
It is a well-known observation that neural nets used in practice are typically sensitive to such perturbations. 
For example a DNN that distinguishes cars from aeroplanes can be very certain that an image that clearly is of a car is in fact a car, 
but after well-chosen perturbations indistinguishable to the human eye, the DNN suddenly believes that the image is an aeroplane.

The simplest model studied in \cite{JSZ}, and, with some modifications, also the model considered here, is a fully connected network in which every hidden layer has width $n$ and the depth is $T_n$.
Asymptotics when $n$ grows to infinity will be in focus. 
The network weights $\Theta_{t,i,j}$ are assumed to be i.i.d. standard Gaussian. 
Unlike in \cite{JSZ}, each node will also have a bias term that is Gaussian with mean 0. 
The given weights constitute a typical initialization before training with data starts. 
The input will always be a binary string and the output activation will always be the sign function. 
In practice this activation is seldom used because it makes training hard, but could be argued to be the most natural one if such considerations were not needed.

To make the setup precise, let $\Theta = \{\Theta_1,\ldots,\Theta_{T_n}\}$ be i.i.d.\ random $n \times n$ matrices, 
where in each $\Theta_t$, the elements $\Theta_t(i,j)$ are i.i.d.\ standard Gaussian. 
Let $B=(\bb_1,\ldots,\bb_{T_n})$ be independent Gaussian $n$-vectors with mean $\bz$ and covariance $c_n n I$, where $c_n$ is assumed to satisfy $1/n \leq c_n \leq \Bar{c} < \iy$. 
Let $\bomega=\bomega_0 \in \{-1,1\}^n$ be the input. 
Then for $t=1,\ldots,T_n$, let 
\[\bomega_t = \sign(\Theta_t\bomega_{t-1} + \bb_t),\]
where the sign operator is taken pointwise.
Finally, writing $\Phi=(\Theta,B,\btheta,b)$, let the output of the DNN be 
\[f^{(n,T_n)}_\Phi(\bomega)=\sign(\btheta^\top\bomega_{T_n} + b),\]
where $\btheta \sim N(0,I_n)$ and $b \sim N(0,c_n n)$.
(i.e.\ $f_{n,T_n}$ is constructed by feeding the penultimate vector $\bomega_{T_n}$ into another independent neuron of the same kind as all the other neurons before it). 
We always assume that $n$ is odd. 
Denote the law of $f_{\Phi}^{(n,T_n)}$ by $\Pro^{(n)}_f$.

In \cite{JSZ}, it was shown that for $c_n=0$, the sequence $\left\{f_{\Phi}^{(n,T_n)}\right\}$ is noise sensitive in a very strong sense:
It is annealed noise sensitive
with respect to any joint distribution $\Qro^{(n)} = \Pro^{(n)}_{\bomega,\bta}$ on the input $\bomega$ and its perturbation $\bta$ for which these are whp not identical or antipodal: $\Qro^{(n)}(\bta \in \{\bomega,-\bomega\}) \to 0$, 
i.e.\ for such $\Qro^{(n)}$, 
\[\lim_{n \to \iy} \Cov_{(f,(\bomega,\bta)) \sim \Pro^{(n)}_f \times \Qro^{(n)}}(f(\bomega),f(\bta)) = 0.\]
If $\Qro^{(n)}$ also satisfies the positive-correlation condition 
\[\Cov_{(\bomega,\bta) \sim \Qro^{(n)}}(h(\bomega),h(\bta)) \geq 0\]
for all Boolean functions $h$,
then this entails that $\left\{f_\Phi^{(n,T_n)}\right\}$ is also quenched noise sensitive with respect to $\Qro^{(n)}$:
for any $\delta >0$, there exists $N<\iy$ such that whenever $n \geq N$, 
\[\Pro^{(n)}_f\left(\Cov_{\Qro^{(n)}}(f(\bomega),f(\bta)) > \delta\right) < \delta.\]

(This quenched noise-sensitivity result was in fact not explicitly stated in \cite{JSZ}, but is immediate from what was done there.)
The positive-correlation condition on $\Qro^{(n)}$ holds very generally, for example if $\Qro^{(n)}$ is symmetric and the (automatically reversible) Markov kernel $K$ given by
\[K(\xx,\yy) = \Qro^{(n)}(\bta=\yy|\bomega=\xx)\]
has only nonnegative eigenvalues, which immediately gives $\Cov_{\Qro^{(n)}}(h(\bomega),h(\bta)) \geq 0$.
A particular class of such distributions is given by taking $(\bomega,\bta)$ as consecutive states of any lazy reversible Markov chain on $\{-1,1\}^n$ at stationarity.
Another such class is given by taking states two steps apart for any reversible stationary Markov chain on $\{-1,1\}^n$. 

In other words, \cite{JSZ} established noise sensitivity for $c_n=0$ for a very broad class of input and perturbation distributions. 
As we shall now see, when adding bias vectors to the model as above, quenched noise sensitivity remains for $T_n$ in a window $[K_n,e^{\rho \sqrt{n}}]$ for $K_n/\log\log n \to \iy$ and some $\rho>0$. 
Indeed it does so in an even stronger form:  
We retain the positive-correlation assumption on $\Qro^{(n)}$, but, in the biased setting considered below, replace the non-antipodal condition by the weaker assumption
$\Qro^{(n)}(\bomega=\bta) \to 0$ as $n \to \iy$.
The exclusion of antipodal input pairs is no longer required because the bias destroys the absorbing antipodal state.

We write $\Pro^{(n)}_{\bomega}$ and $\Pro^{(n)}_{\bta}$ for the two marginals of $\Qro^{(n)}$.
Before stating the main result, we need the following.

\begin{lemma} \label{lemma:properties_of_g}
    Consider
    \[g_c(x) = \frac{2}{\pi}\arcsin\sqrt{\frac{x}{1+c}}, \,\, 0 \leq c \leq \Bar{c}, \,\, 0 \leq x \leq 1.\]
    Then $g_c$ has a unique fixed point $v_c$ in $(0,1)$, for which $\tfrac{4}{\pi^2(1+c)} \leq v_{c} \leq 1/2$, and for any sufficiently small $\ep$, there exists $\alpha>0$ such that $g_c(x) \geq x+\alpha$ on $[\ep,v_c-\ep]$ and $g_c(x) \leq x-\alpha$ on $[v_c+\ep,1-\ep]$ uniformly over $c \in [0,\Bar{c}]$.
\end{lemma}

\begin{theorem} \label{thm:neither}
Let $\Qro^{(n)}$ and $\Phi=(\Theta,B,\btheta,b)$ be as above, and let $g_c$ be as in Lemma \ref{lemma:properties_of_g}. 
Assume that $K_n/\log \log n \to \iy$. Then there exists $\rho>0$ such that whenever
\[K_n \leq T_n \leq e^{\rho \sqrt{n}},\]
and $\Pro^{(n)}_f$ is the resulting law of $f_\Phi^{(n,T_n)}$,
the following conclusions hold.
\begin{itemize}
\item[(i)] {\bf Quenched noise sensitivity:} We have
\[\Cov_{\Qro^{(n)}}(f(\bomega),f(\bta)) \stackrel{\Pro^{(n)}_f}{\longrightarrow} 0\]
as $n \to \iy$,
in particular $\{f_\Phi^{(n,T_n)}\}$ is quenched noise sensitive with respect to $\Qro^{(n)}$.

\item[(ii)] {\bf Annealed correlation:} If $c_n = c > 0$, then, with $v_c$ being the unique fixed point of $g_c$,
\begin{align*} 
\lim_{n \to \iy}\Cov_{\Pro^{(n)}_f \times \Qro^{(n)}}(f(\bomega),f(\bta)) 
& = \lim_{n \to \iy}\Cov_{\Pro^{(n)}_f}\left(\E_{\Qro^{(n)}}\left[f(\bomega)\right],\E_{\Qro^{(n)}}\left[f(\bta)\right]\right) \\
&= 1-2v_c.
\end{align*}

\item[(iii)] {\bf Variance of the conditional output mean:} If, in addition to the conditions in (ii), $\sum_{\xx \in \{-1,1\}^n}\Pro^{(n)}_{\bomega}(\xx)^2 \to 0$ as $n \to \iy$, then
\[\lim_{n \to \iy}\Var_{\Pro^{(n)}_f}\left(\E_{\Pro^{(n)}_{\bomega}}\left[f(\bomega)\right]\right) = 1-2v_c.\]
\end{itemize}
\end{theorem}

\section{Proofs}

The following lemma is a different formulation of Lemma 3.1 in \cite{JSZ}.

\begin{lemma} \label{lemma:diff}

Let $X$ and $Y$ be independent Gaussian random variables with mean 0 and variances $a$ and $b$ respectively. Then 
\[\Pro(\sign(X-Y) \neq \sign(X+Y)) = \frac{2}{\pi} \arcsin\sqrt{\frac{b}{a+b}}.\]

\end{lemma}

A consequence of Lemma \ref{lemma:diff} is that if $\bomega$ and $\bta$ are two bit strings of dimension $n$ that differ at $nx$ coordinates, then 
\[\Pro(\sign((\Theta_t)_{i,\cdot}\bomega+b_t(i))\ \neq \sign((\Theta_t)_{i,\cdot} \bta + b_t(i))) = \frac{2}{\pi}\arcsin{\sqrt{\frac{x}{1+c_n}}} = g_{c_n}(x)\]
for $g_c$ as in Lemma \ref{lemma:properties_of_g}, where $(\Theta_t)_{i,\cdot}$ denotes the $i$th row of $\Theta_t$.

{\em Proof of Lemma \ref{lemma:properties_of_g}:}
It is obvious that $g_c$ is nonnegative, increasing, and of order $\sqrt{x}$ as $x \ra 0$.
Also,
\[1-g_c(1-x) \geq g_c(x) \vee \frac{2}{\pi}\arctan\sqrt{c},\]
where the first inequality follows from $g_c(x)+g_c(1-x)\le 1$.
Moreover, 
\[g_c''(x) = -\frac{1+c-2x}{2\pi(x(1+c-x))^{3/2}},\]
which shows that $h_c(x):=g_c(x)-x$ is concave on $[0,a]$ and convex on $[a,1]$, where $a=1 \wedge (1+c)/2$.
Let 
\[0<\ep < \min\left\{\frac{2}{\pi^2(1+\Bar{c})},\frac{1}{4}\right\}.\]
Then
\[-h_c(1-\ep) \geq h_c(\ep) \geq \frac{2}{\pi}\sqrt{\frac{\ep}{1+c}}-\ep  
\geq \frac{2}{\pi}\sqrt{\frac{\ep}{1+\Bar{c}}}-\ep > 0.\]
We claim that $h_c$ has
a unique zero $v_{c} \in (0,1)$, which satisfies $v_c \in [\ep,1-\ep]$, $h_c>0$ on $(0,v_{c})$, and $h_c<0$ on $(v_{c},1)$.
That $v_c \in [\ep,1-\ep]$ follows from the inequalities displayed. To see the other two conclusions, use that $h_c(a) \geq 0$ or $h_c(a) <0$.
If the former holds, then concavity and $h_c(\ep)>0$ prove that $h_c$ has no zero on $[\ep,a)$, and convexity and $h_c(1-\ep)<0$ prove that
there is a unique zero in $[a,1-\ep]$.
If the latter holds, concavity and $h_c(\ep)>0$ prove a unique zero in $[\ep,a]$, and convexity and $h_c(1-\ep)<0$ prove the absence of a zero in $[a,1-\ep]$.

Knowing that $v_c$ is unique and using that $\arcsin u \geq u$ and $g_c(1/2) \leq g_0(1/2)=1/2$, it is now easy to show that $\tfrac{4}{\pi^2(1+c)} \leq v_{c} \leq 1/2$.
This also means, by the choice of $\ep$, that the two intervals $[\ep,v_c-\ep]$ and $[v_c+\ep,1-\ep]$ are nonempty.
Since $v_c$ is the unique solution of $\Psi(x)=c$ for a continuous strictly decreasing function $\Psi$, $v_c$ is continuous in $c$.
Hence the sets
\[A_-=\{(c,x): 0 \leq c \leq \Bar{c},\,\ep \leq x \leq v_c-\ep\}\]
and
\[A_+=\{(c,x): 0 \leq c \leq \Bar{c},\,v_c+\ep \leq x \leq 1-\ep\}\]
are compact.
Moreover, $h_c(x) > 0$ on $A_-$ and $h_c(x) < 0$ on $A_+$.
Hence
\[\alpha:=\min\left\{\min_{A_-}h_c(x),\min_{A_+}(-h_c(x))\right\} > 0,\]
and $g_{c}(x) \geq x+\alpha$ on $[\ep,v_{c}-\ep]$
and $g_{c}(x) \leq x-\alpha$ on $[v_{c}+\ep,1-\ep]$ uniformly over $c \in [0,\Bar{c}]$.
\hfill $\Box$

\smallskip

{\em Proof of Theorem \ref{thm:neither}}:
Let $\bomega=\bomega_0$ and $\bta = \bta_0$ be two inputs into the DNN $f_{\Phi}^{(n,T_n)}$ defined above.
Let $W_t$ be the set of coordinates where $\bomega_t$ and $\bta_t$ differ, and let $D_t = |W_t|/n$.
Conditionally on $D_{t-1} = x$, the disagreement indicators in layer $t$ are i.i.d.\ Bernoulli with parameter $g_{c_n}(x)$, so that in particular $W_t$ is invariant under permutations and $nD_t \mid D_{t-1}=x \sim Bin(n,g_{c_n}(x))$.

On $[1/n,\ep]$, a lower bound $g_{c_n}(x) \geq 2C\sqrt{x}$ applies uniformly.
By standard Chernoff bounds, there then exists $\rho>0$ such that for $x \in [1/n,\ep]$,
\[\Pro^{(n)}_f(D_{t+1} \leq C\sqrt{x} \mid D_t =x) \leq e^{-3\rho \sqrt{n}}.\]
Hence, for $n$ large and $0< D_0 < \ep$, $D_t$ exceeds $\ep$ within $O(\log \log n)$ steps with probability
at least $1-e^{-2\rho\sqrt{n}}$, by a union bound over layers.
By identical arguments, if $D_0>1-\ep$, then $D_t$ goes below $1-\ep$ within time $O(\log\log n)$ with probability at least $1-e^{-2\rho\sqrt{n}}$.

Using further Chernoff bounds, union bounds over layers, and Lemma \ref{lemma:properties_of_g}, it follows that for $t_0(n)/\log\log n \to \iy$ and $t_1(n)=e^{\rho \sqrt{n}}$ and sufficiently small $\varphi >0$,
\[\Pro^{(n)}_f\left(\exists t \in [t_0(n),t_1(n)]: D_t \not \in ((1-\varphi)v_{c_n},(1+\varphi)v_{c_n})\right) < e^{-\rho \sqrt{n}}.\]
Hence for $T_n \in [t_0(n),t_1(n)]$, 
\[\Pro^{(n)}_f \left(D_{T_n} \not \in (1 \pm \varphi)v_{c_n}\right) < e^{-\rho \sqrt{n}}.\]
From this we then get for sufficiently large $n$
\begin{align*}
    \Pro^{(n)}_f(f(\bomega) \neq f(\bta)) &= \E_{\Pro^{(n)}_f}\left[\Pro^{(n)}_f(f(\bomega) \neq f(\bta) \mid D_{T_n})\right] \\
&\in [-e^{-\rho \sqrt{n}}+g_{c_n}((1-\varphi)v_{c_n}), e^{-\rho \sqrt{n}} + g_{c_n}((1+\varphi)v_{c_n})] \\
&\in (1 \pm \varphi)v_{c_n}.
\end{align*}

By the symmetry of $\Phi$, $\E_{\Pro^{(n)}_f}[f(\bomega)] = 0$, and it follows that
\[\Cov_{\Pro^{(n)}_f}(f(\bomega),f(\bta)) \in 1- (1 \pm \varphi)2v_{c_n}.\]
Since $\Qro^{(n)}(\bomega=\bta) \rightarrow 0$, it also follows that for sufficiently large $n$,
\[(\Pro^{(n)}_f \times \Qro^{(n)})\left(f(\bomega) \neq f(\bta)\right) \in (1 \pm 2\varphi)v_{c_n}\]
and
\[\Cov_{\Pro^{(n)}_f \times \Qro^{(n)}}\left(f(\bomega),f(\bta)\right) \in 1- (1 \pm 2\varphi)2v_{c_n}.\]
Next observe that
\begin{align*} 
\Cov_{\Pro^{(n)}_f}\left(\E_{\Qro^{(n)}}\left[f(\bomega)\right],\E_{\Qro^{(n)}}\left[f(\bta)\right]\right) &= \sum_{\xx,\yy \in \{-1,1\}^n}\Pro^{(n)}_{\bomega}(\xx)\Pro^{(n)}_{\bta}(\yy)\Cov_{\Pro^{(n)}_f}(f(\xx),f(\yy)) \\
& \geq 1-(1+\varphi)2v_{c_n}.
\end{align*}
The same calculation, replacing $\Pro^{(n)}_{\bta}$ with $\Pro^{(n)}_{\bomega}$ and under the additional assumption $\sum_{\xx \in \{-1,1\}^n}\Pro^{(n)}_{\bomega}(\xx)^2 \to 0$, gives
\[\Var_{\Pro^{(n)}_f}\left(\E_{\Pro^{(n)}_{\bomega}}\left[f(\bomega)\right]\right) \in 1-(1\pm 2\varphi)2v_{c_n},\]
proving (iii).

By the positive-correlation condition on $\Qro$ and the covariance decomposition formula,
\begin{align*} 
0 & \leq \E_{\Pro^{(n)}_f}\left[\Cov_{\Qro^{(n)}}(f(\bomega),f(\bta))\right] \\
&=  \Cov_{\Pro^{(n)}_f \times \Qro^{(n)}}\left(f(\bomega),f(\bta)\right) - \Cov_{\Pro^{(n)}_f}\left(\E_{\Qro^{(n)}}\left[f(\bomega)\right],\E_{\Qro^{(n)}}\left[f(\bta)\right]\right) \\
&\leq 3\varphi.
\end{align*}
By Markov's inequality, it follows for any $\delta>0$, by taking $\varphi>0$ sufficiently small and $n$ sufficiently large,
\[\Pro^{(n)}_f\left(\Cov_{\Qro^{(n)}}(f(\bomega),f(\bta)) > \delta\right) < \delta.\]
This proves (i), and in combination with the previous lower bound on the second term, it also proves (ii). 
\hfill $\Box$

\end{document}